\documentclass[mnsc,nonblindrev]{informs3_hide}
\OneAndAHalfSpacedXI 

\usepackage[english]{babel}
\usepackage[autostyle, english = american]{csquotes}
\MakeOuterQuote{"}
\usepackage[colorlinks,linkcolor=blue, citecolor=blue]{hyperref}
\usepackage[normalem]{ulem}

\usepackage{cleveref}
\crefname{subsection}{section}{subsections}
\usepackage{eqnarray}
\usepackage{algorithm,algpseudocode}
\usepackage{amsmath, bbm, enumitem, xparse, bm, lipsum}
\usepackage{amssymb}
\usepackage{comment}
\usepackage{mathtools}
\usepackage{subcaption}
\usepackage{booktabs}

\newcommand{\ALG}{\mathsf{ALG}}
\newcommand{\OPT}{\mathsf{OPT}}

\newcommand{\Var}{\mathrm{Var}}

\newcommand{\R}{\mathbb{R}}
\newcommand{\E}{\mathbb{E}}

\newcommand{\Reg}{\operatorname{Reg}}

\DeclarePairedDelimiter{\norm}{\lVert}{\rVert}

\Crefname{assumption}{Assumption}{Assumptions}
\crefname{assumption}{assumption}{assumptions}
\newcommand{\Nmax}{N_{\max}}
\newcommand{\cbar}{\bar c}
\newcommand{\qhat}{\widehat q}
\newcommand{\Z}{\mathbb Z}
\newcommand{\zhat}{\widehat z}
\newcommand{\proj}{\Pi}

\NewDocumentEnvironment{myproof}{o}
{\IfNoValueTF{#1}{\paragraph{{Proof.} }} {\paragraph{{#1.} }} }
{\hfill$\qed$}

\usepackage[dvipsnames]{xcolor}

\usepackage{makecell}
\usepackage[table]{xcolor}
\definecolor{oursbg}{HTML}{F1EFFE}
\newcommand{\ourscell}[1]{\cellcolor{oursbg}#1}
\usepackage{wrapfig}
\usepackage{xparse}

\usepackage{natbib,bbm,multirow,multicol}
 \bibpunct[, ]{(}{)}{,}{a}{}{,}%
 \def\bibfont{\small}%
\TheoremsNumberedThrough     
\ECRepeatTheorems

\EquationsNumberedThrough    

\allowdisplaybreaks

\begin{document}


\RUNAUTHOR{Zong, Wang, Jiang}

\RUNTITLE{CERO: Adaptive Rollout Optimization}

\TITLE{
CERO: Where and When to Allocate Rollouts for RL Post-Training
}

\ARTICLEAUTHORS{%
\AUTHOR{Yiming Zong\textsuperscript{a},\quad Yige Wang\textsuperscript{a}, \quad Xing Hu\textsuperscript{b}, \quad Jiashuo Jiang\textsuperscript{a}, \quad Zuo-Jun Max Shen\textsuperscript{c}}

\AFF{
\textsuperscript{a} Department of Industrial Engineering \& Decision Analytics, Hong Kong University of Science and Technology \\
\textsuperscript{b} Innovation and Information Management, The University of Hong Kong \\
\textsuperscript{c} Faculty of Engineering \& Faculty of Business and Economics, The University of Hong Kong
}
}

\ABSTRACT{
Adaptive rollout methods for group-relative reinforcement learning typically allocate a fixed per-update budget across prompts. We instead study how to coordinate a finite rollout budget over the entire training horizon. We formulate this problem using a concave surrogate utility of cumulative prompt exposure and introduce CERO, an online primal–dual scheduler for prompt admission and budget pacing. In our experiments, each admitted prompt receives a fixed-size response group. CERO instead adapts which prompts are selected, how often they are revisited across rounds, and how many groups are generated in each round. A compact Fenchel representation linearizes the dependence on cumulative exposure, while projected online gradient descent updates prompt-specific supporting slopes and a shared budget price using reward-variation feedback and budget deviations. We establish pathwise guarantees for the surrogate allocation objective against fixed-rate and same-path time-varying benchmarks, with explicit terms for proxy discrepancy and rate variation. Under matched training-response budgets, CERO attains the highest avg@16 macro-average on each of three backbones across five mathematical reasoning benchmarks. Mechanistic analyses link CERO's prompt choices to within-group reward contrast, while multi-seed ablations show gains from adaptive pacing over both uniform and preset spending schedules.
}

\KEYWORDS{Reinforcement Learning, LLM Post-training, Online Resource Allocation, Rollout Optimization}


\maketitle

\section{Introduction}
Reinforcement learning with verifiable rewards has become a standard approach for post-training reasoning models. Within this paradigm, recent work has increasingly focused on adaptive rollout allocation to improve training efficiency. Standard GRPO-style training typically assigns the same number of responses to each selected prompt \citep{shao2024deepseekmath,guo2025deepseek}, even though the resulting groups can differ substantially in informativeness. In particular, all-correct or all-incorrect groups provide no within-group reward contrast. Recent methods therefore adapt how rollout budgets are distributed across prompts: KnapsackRL \citep{li2025knapsack} allocates a fixed exploration budget according to estimated task value, VIP \citep{nguyen2026adaptive} allocates minibatch rollouts to reduce expected gradient variance, and HORA \citep{wang2026spend} redistributes a batch budget according to marginal hit utility.
Collectively, these methods primarily adapt \emph{where} rollouts are allocated within a given round-level budget. We focus on a complementary temporal dimension that arises when rollout decisions are coupled by a shared finite training budget:

{\centering
\textit{Given a finite rollout budget, when should we spend more, and when less, across training?}
\par}


This horizon-level problem remains nontrivial even when each admitted prompt receives a fixed-size response group. The scheduler must still decide which prompts to admit, how often to revisit them across rounds, and how much of the remaining budget to spend at different training stages. Rollouts spent now cannot be reassigned later, and repeated exposure to the same prompt may yield diminishing marginal utility. We capture this effect with a concave utility over cumulative prompt exposure. These couplings motivate a resource allocation formulation over prompts and rounds.

However, this finite-horizon formulation is not directly separable across rounds: each prompt's cumulative allocation appears inside a nonlinear concave utility, and the global budget further couples spending decisions over time. Our key methodological contribution is to adapt Fenchel dualization and projected online gradient descent (OGD) to this rollout optimization problem. The Fenchel representation linearizes dependence on cumulative exposure, reducing the coupled horizon objective to tractable round-level allocation problems. Projected OGD then adapts prompt-specific supporting slopes from reward feedback and a shared budget price from aggregate allocation deviations.

These dual signals play complementary roles in CERO: prompt-specific slopes guide prompt prioritization, while the shared price regulates budget pacing across rounds. A prompt-level Beta state provides a predictive signal of future within-group reward variation, while the dual variables and the remaining budget are carried forward across rounds, enabling coordination throughout training. This online structure admits pathwise guarantees for the surrogate allocation objective under fixed rates and a same-path time-varying benchmark. Empirically, CERO achieves the highest avg@16 macro-average on all three evaluated backbones under matched training-response budgets. 
Matched controls and mechanistic analyses further connect these gains to prompt-aware assignment and horizon-level budget adaptation, with CERO outperforming each ablated variant in multi-seed average performance.

\begin{figure}[t]
    \centering
    \includegraphics[width=\linewidth]{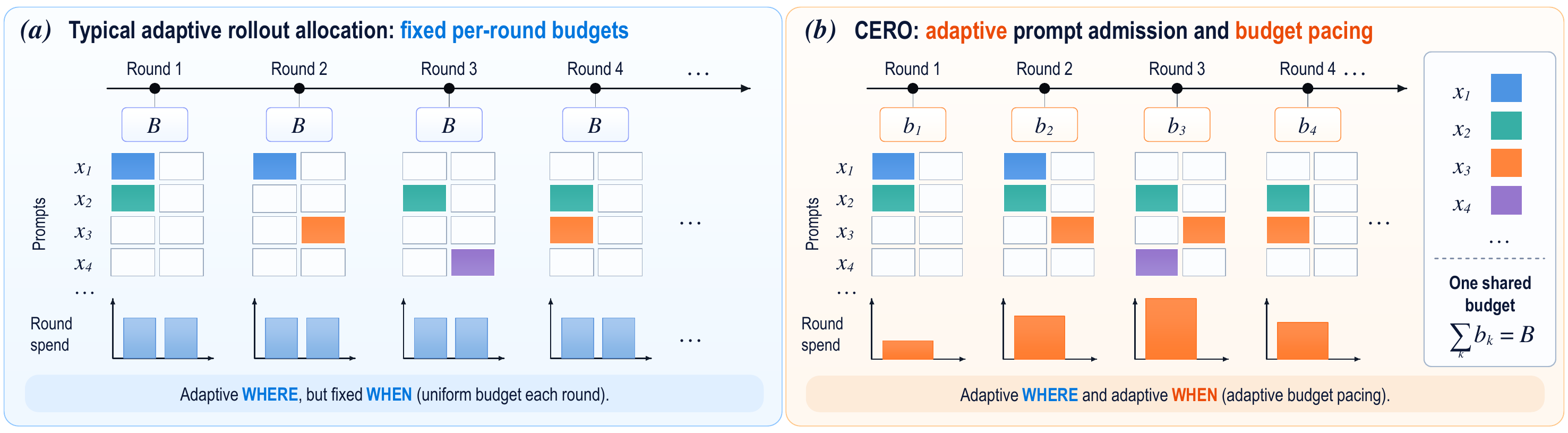}
    \setlength{\abovecaptionskip}{1pt}
    \caption{Schematic comparison of rollout optimization.
(a) Typical adaptive methods redistribute a fixed per-round budget across prompts.
(b) CERO coordinates prompt admission and per-round budget pacing under a shared finite training budget, with admitted prompts receiving fixed-size response groups.}
    \label{fig:pipeline}
\end{figure}

Our contributions are threefold:
\begin{itemize}

\item We formulate finite-horizon rollout optimization as a resource-allocation problem with a concave utility over cumulative prompt exposure, and derive an online primal–dual scheduler using a compact Fenchel representation and projected OGD. The resulting scheduler coordinates prompt admission, repeated exposure across rounds, and per-round spending under a shared finite budget.

\item We establish pathwise guarantees for the surrogate allocation objective under fixed rates and a same-path time-varying benchmark. The bounds explicitly separate online optimization error, reward-proxy discrepancy, and temporal rate variation.

\item We evaluate CERO on three backbones and five mathematical reasoning benchmarks under the same rollout budgets. CERO obtains the highest avg@16 macro-average on each backbone. Matched controls and multi-seed ablations further connect these gains to prompt-aware assignment and horizon-level coordination, with CERO outperforming fixed-pacing, preset-pacing, and direct marginal-utility variants.

\end{itemize}

\paragraph{Relation to prior work.}
CERO lies at the intersection of three related directions: difficulty-aware sampling and filtering, adaptive rollout allocation, and online resource allocation. The first focuses on identifying and prioritizing informative prompts, the second studies how limited rollout budgets should be allocated across prompts, and the third provides a broader framework for coordinating sequential decisions under shared resource constraints. Appendix~\ref{app:lr} provides a detailed review of these three lines of research.

\section{Finite-Horizon Rollout Optimization}
\label{sec:optimization}

\subsection{Problem Formulation}
\label{sec:plan}

We consider a finite training process of $K$ rounds over a fixed pool $\mathcal X=\{x_1,\ldots,x_M\}$ of $M$ prompts.  Let $N_{i,k}$ denote the number of units assigned to prompt $i$ in round $k$, where one allocation unit corresponds to a fixed quantum of $g\in\mathbb Z_{>0}$ generated responses. Define its cumulative exposure by $z_{i,k}:=\sum_{t=1}^{k}N_{i,t},~ z_{i,0}=0$. Let $N_{\max},C,B\in\mathbb Z_{\ge0}$ be the per-prompt per-round cap, the aggregate per-round capacity, and the total rollout budget. Our goal is to maximize the total utility of cumulative prompt exposure by allocating a finite rollout budget across prompts and rounds:
\begin{equation}
\begin{aligned}
\max_{\mathbf N\in\{0,\ldots,N_{\max}\}^{M\times K}}\quad
& \sum_{i=1}^M U_i(z_{i,K})\\
\text{s.t.}\quad
& \sum_{i=1}^M N_{i,k}\le C,
&& k\in[K],\\
& \sum_{k=1}^K\sum_{i=1}^M N_{i,k}\le B.
\end{aligned}
\label{eq:rollout-problem}
\end{equation}
Here $U_i$ is a nondecreasing concave utility of cumulative exposure, specified
in the next subsection. We assume
\begin{equation}
0<B\le KC,\quad C\le MN_{\max}.
\label{eq:capacity-assumption}
\end{equation}
Denote the feasible set in \cref{eq:rollout-problem} by
$\mathcal N_{B,C}$.
We distinguish allocation units from generated-response counts using the superscript resp:
$N_{i,k}^{\mathrm{resp}}=gN_{i,k}$,
$N_{\max}^{\mathrm{resp}}=gN_{\max}$,
$C^{\mathrm{resp}}=gC$, and
$B^{\mathrm{resp}}=gB$.
Finally, define
\[
Z:=KN_{\max},
\qquad
\bar B:=\frac{B}{K},
\qquad
R:=\frac{B}{M},
\]
where $Z$ is the maximum per-prompt exposure over training, $\bar B$ is the
average per-round budget, and $R$ is the equal-share exposure scale.

\subsection{Value-capture utility}
\label{sec:utility}
We use a concave surrogate utility to model diminishing returns to cumulative prompt exposure. Each prompt has a latent value-capture rate \(q_i\in[0,1/4]\), for which Section 3 provides a reward-variation proxy. We instantiate the utility with the normalized exponential form
$$
U(q,z):=1-\exp\left(-4\lambda q\frac{z}{R}\right),
\quad q\in[0,1/4],\quad z\in[0,Z],
$$
where $\lambda>0$ controls saturation and $z/R$ measures exposure relative to the equal-share horizon scale. By construction, the utility is nondecreasing and concave in $z$, encoding diminishing marginal utility with exposure. The marginal utility of an additional allocation unit satisfies
\begin{equation}
0\le \partial_zU(q,z)
=\frac{4\lambda q}{R}\exp\left(-4\lambda q\frac{z}{R}\right)
\le \frac{\lambda}{R}=:\cbar.
\label{eq:slope-bound}
\end{equation}
This bound yields the compact slope domain $[0,\cbar]$ used in the Fenchel representation below.
For comparison, consider a full-information oracle that knows
$\bm q=(q_1,\ldots,q_M)$ in advance and chooses the feasible allocation maximizing total value capture:
\begin{equation}
\OPT_K^{\mathrm{stat}}(\bm q)
:=\max_{\mathbf N\in\mathcal N_{B,C}}
\sum_{i=1}^M U(q_i,z_{i,K}).
\label{eq:static-opt}
\end{equation}
Since the objective depends on the allocation history only through $z_{i,K}$, the oracle problem is determined by the terminal exposure vector. Appendix~\ref{app:full-information-oracle} gives its optimal allocation. 

\subsection{Compact Fenchel representation}
\label{sec:fenchel}
To simplify notation, let $U_i(z):=U(q_i,z)$. Using the uniform slope bound in \cref{eq:slope-bound}, define the restricted concave conjugate
$$
U_i^*(\theta):=\inf_{0\le s\le Z}\{\theta s-U_i(s)\},
\quad \theta\in[0,\cbar].
$$
The interval $[0,\cbar]$ contains a supporting slope at every exposure level, so the representation remains exact on this compact domain. By concavity,
$$
U_i(z)=\min_{0\le\theta\le\cbar}\{\theta z-U_i^*(\theta)\},
\qquad 0\le z\le Z.
$$
Thus, the nonlinear dependence on cumulative exposure can be represented exactly through affine functions of $z$.
To capture the scarcity of the shared budget across rounds, let $\mu\in[0,\cbar]$ denote the corresponding dual price. For a single round, let $\bm n=(n_1,\ldots,n_M)$ denote the allocation vector and define the continuous feasible region
$$
\mathcal P_C:=\left\{\bm n\in\R^M:0\le n_i\le\Nmax,\ \sum_i n_i\le C\right\}.
$$
Since $C$ and $N_{\max}$ are integers, $\mathcal P_C$ is integral, so the continuous relaxation is exact for the round-level linear problem. Let $\bm\theta=(\theta_1,\ldots,\theta_M)\in[0,\cbar]^M$ collect the prompt-specific supporting slopes. Combining the Fenchel representation with the shared budget price $\mu$ makes the horizon-level upper bound separable across rounds, yielding the per-round dual objective
$$
L(\bm\theta,\mu)
=\bar B\mu+
\max_{\bm n\in\mathcal P_C}\sum_{i=1}^M(\theta_i-\mu)n_i
-\frac1K\sum_{i=1}^M U_i^*(\theta_i),
$$
By weak duality, $\OPT_K^{\mathrm{stat}}(\bm q)\le K L(\bm\theta,\mu)$ for all $(\bm\theta,\mu)$ in the stated domains. 
Let $d_i:=\theta_i-\mu$, which compares prompt $x_i$'s supporting slope against the shared budget price.


\section{Reward-Variation Proxy}
We use predicted within-group reward variance as a proxy for the latent value-capture rate in Section ~\ref{sec:utility}. This quantity characterizes expected reward contrast in a fresh response group, rather than directly estimating the downstream improvement from a policy update.
Let $r(x,y)\in\{0,1\}$ be the verifiable reward. At the beginning of round $k$, we maintain an auxiliary Beta state for prompt $x_i$:
$$
P_{i,k}\mid\mathcal H_i^k\sim\mathrm{Beta}(a_{i,k},b_{i,k}),
\quad a_{i,k},b_{i,k}>0,
$$
where $\mathcal H_i^k$ denotes the reward history of prompt $i$. This state provides a compact predictive summary of the prompt's observed reward history.
Let $G_{i,k}:=N_{i,k}^{\mathrm{resp}}=gN_{i,k}$ denote the number of responses generated for prompt $x_i$ in round $k$. Let $S_{i,k}$ and $F_{i,k}$ denote the numbers of successful and failed responses, and set $S_{i,k}=F_{i,k}=0$ if there is no allocation to this prompt. After round $k$, we update the Beta state with forgetting factor $\rho\in(0,1]$:
\begin{align}
a_{i,k+1}=a_0+\rho(a_{i,k}-a_0)+S_{i,k},~
b_{i,k+1}=b_0+\rho(b_{i,k}-b_0)+F_{i,k},
\label{eq:discounted-beta}
\end{align}
where $a_0,b_0>0$ are the prior parameters. When $\rho=1$, the update reduces to standard cumulative Beta--Bernoulli updating. When $\rho<1$, older evidence is gradually discounted toward the prior, allowing the state to track changes in the policy.
We use predicted reward variation as a proxy for the latent value-capture rate:
\begin{equation}
\qhat_{i,k}:=\E[P_{i,k}(1-P_{i,k})\mid\mathcal H_i^k]
=\frac{a_{i,k}b_{i,k}}{(a_{i,k}+b_{i,k})(a_{i,k}+b_{i,k}+1)}
\in[0,1/4].
\label{eq:qhat}
\end{equation}

\begin{proposition}[Predicted within-group reward variance]
\label{prop:group-variance}
For any group size $G\ge2$, suppose that conditional on $(P_{i,k},\mathcal H_i^k)$ the next $G$ binary rewards are i.i.d. Bernoulli$(P_{i,k})$. Let $\bar r_G:=G^{-1}\sum_{j=1}^G r^{(j)}$ and
$
\widehat{\mathrm{Var}}_G(r)
:=
\frac{1}{G-1}
\sum_{j=1}^{G}
\left(r^{(j)}-\bar r_G\right)^2.
$
Then it holds that
$$
\E[\widehat{\operatorname{Var}}_G(r)\mid\mathcal H_i^k]=\qhat_{i,k}.
$$
\end{proposition}
Thus, $\qhat_{i,k}$ is the predictive expected within-group reward variance.
Let $m_{i,k}:=\E[P_{i,k}\mid\mathcal H_i^k]$ denote the predictive mean and we can obtain
\[
\qhat_{i,k}
=
m_{i,k}(1-m_{i,k})
-
\Var(P_{i,k}\mid\mathcal H_i^k).
\]
Therefore, $\qhat_{i,k}$ is largest when the predicted success rate is near
$1/2$, but is reduced when the Beta state is more uncertain.

\begin{proposition}[Fixed-sample stationary calibration]
\label{prop:beta-calibration}
Let $Y_1,\ldots,Y_n$ be i.i.d.\ Bernoulli$(p)$ observations with fixed sample size $n$. Define
$
a:=a_0+\sum_{j=1}^n Y_j,~
b:=b_0+n-\sum_{j=1}^n Y_j,~
s_0:=a_0+b_0,
$
and let
$\qhat_n:=ab/[(a+b)(a+b+1)]$ and $q:=p(1-p)$.
Then it holds that
\begin{equation}
\E|\qhat_n-q|
\le
\frac{\sqrt n}{2(n+s_0)}
+\frac{s_0}{n+s_0}
+\frac{1}{4(n+s_0+1)}.
\label{eq:beta-calibration-bound}
\end{equation}
Consequently, $\qhat_n$ is consistent at rate $O(n^{-1/2})$ under a stationary Bernoulli model  with a fixed, nonadaptive sample size
\end{proposition}
\Cref{prop:beta-calibration} concerns fixed-sample stationary calibration. Our allocation guarantees instead hold pathwise for arbitrary realized proxy sequences, with the bounds explicitly accounting for proxy error relative to the target rates.

\section{CERO Algorithm}
To solve the finite-horizon rollout-allocation problem, we develop CERO, an online primal--dual scheduler summarized in \Cref{alg:cero}. At each round, CERO maintains normalized prompt-specific supporting slopes and a normalized
shared budget price:
\[
\theta_i^k=\cbar\,\vartheta_i^k,
\quad
\mu^k=\cbar \,\nu^k,
\quad
\vartheta_i^k,\nu^k\in[0,1].
\] 
After generating the selected responses and observing their rewards, CERO updates the predictive state and dual variables for the next round.

\paragraph{Round-level allocation}
Since $\theta_i^k-\mu^k=\cbar(\vartheta_i^k-\nu^k)$, define the normalized
margin
$
d_{i,k}:=\vartheta_i^k-\nu^k.
$
A positive margin indicates that prompt $x_i$'s supporting slope exceeds the
shared budget price. CERO first computes the capacity-constrained virtual allocation
\begin{equation}
\widetilde{\bm N}_k\in
\arg\max_{\bm n\in\mathcal P_C\cap\Z^M}
\sum_{i=1}^M d_{i,k}n_i.
\label{eq:round-proposal}
\end{equation}
This problem has a simple exact solution: prompts with positive margins are ranked in decreasing order, with up to $N_{\max}$ units assigned to each while capacity remains. CERO then allocates rollout units in this order up to the remaining global budget, which is updated as $B_{k+1}=B_k-\sum_i N_{i,k}.$

Since the round-level objective is linear in the allocation units, the ranked solution assigns selected prompts the per-prompt cap, except possibly for a residual allocation. In Section~6, we use \(N_{\max}=2\) and choose the round capacity and total budget as multiples of \(N_{\max}\), so every selected prompt receives one eight-response group. CERO therefore adapts which prompts receive groups, how often they are revisited across rounds, and how many groups are generated per round.

\paragraph{Reward feedback and supporting exposure}
After observing the rewards from round $k$, CERO updates the Beta state using \cref{eq:discounted-beta} and recomputes the reward-variation proxy $\qhat_{i,k+1}$ using \cref{eq:qhat}. Let
$
\widehat U_{i,k+1}(z):=U(\qhat_{i,k+1},z),~
c_{i,k+1}:=\frac{4\lambda\qhat_{i,k+1}}{R}.
$
For a supporting slope $\theta$, the corresponding exposure is
\[
\hat z^*_{i,k+1}(\theta)
\in
\arg\min_{0\le z\le Z}
\left\{\theta z-\widehat U_{i,k+1}(z)\right\}.
\]
For our exponential utility, this minimizer has the closed
form $\hat z^*_{i,k+1}(\theta)=\hat z^*(\theta;c_{i,k+1})$, where
\begin{equation}
\hat z^*(\theta;c)=
\begin{cases}
0, & c=0\text{ or }\theta\ge c,\\[1mm]
Z, & c>0\text{ and }0\le\theta\le ce^{-cZ},\\[1mm]
\dfrac1c\log\dfrac{c}{\theta},
& c>0\text{ and }ce^{-cZ}<\theta<c.
\end{cases}
\label{eq:zstar-closed}
\end{equation}

\paragraph{Dual updates}
Let $G_B:=\max\{\bar B,C-\bar B\}$ denote the normalization scale for the budget-price update. CERO applies the normalized projected updates
\begin{align}
\label{eq:theta-mu-update-normalized}
\vartheta_i^{k+1}
=\proj_{[0,1]}\left[
\vartheta_i^k-
\gamma_\vartheta
\frac{\widetilde N_{i,k}-\zhat_{i,k+1}^*(\cbar\vartheta_i^k)/K}{\Nmax}
\right],~
\nu^{k+1}
=\proj_{[0,1]}\left[
\nu^k-
\gamma_\nu
\frac{\bar B-\sum_i\widetilde N_{i,k}}{G_B}
\right]
\end{align}
The prompt-specific update compares the virtual allocation $\widetilde N_{i,k}$ with the per-round supporting exposure $\hat z^*_{i,k+1}(\cbar\vartheta_i^k)/K$: allocating above this level decreases the supporting slope, while allocating below it increases the slope. The shared price increases when total virtual allocation exceeds $\bar B$ and decreases when it falls below $\bar B$. Both updates use the virtual allocations $\widetilde{\bm N}_k$ rather than the budget-clipped allocations. Overall, the prompt-specific dual variables adapt relative prompt priority, while the shared price regulates budget pacing across rounds.

\begin{algorithm}[t]
\caption{CERO: Cross-Epoch Adaptive Rollout Optimization}
\label{alg:cero}
\small
\begin{algorithmic}[1]
\Require Prompt pool $\{x_i\}_{i=1}^{M}$; horizon $K$; global budget $B$; round capacity $C$; per-prompt cap $N_{\max}$.
\State Set $R=B/M$, $Z=K\Nmax$, $\cbar=\lambda/R$, $\bar B=B/K$, and $G_B=\max\{\bar B,C-\bar B\}$
\State Initialize $B_1=B$, $(a_{i,1},b_{i,1})=(a_0,b_0)$, $\qhat_{i,1}$ by \eqref{eq:qhat}, $\vartheta_i^1=4\qhat_{i,1}$, and $\nu^1\in[0,1]$
\For{$k=1,\ldots,K$}
    \State Compute margins $d_{i,k}=\vartheta_i^k-\nu^k$ and virtual proposal $\widetilde{\bm N}_k$ by the top-score solution of \eqref{eq:round-proposal}
    \State Execute proposals in decreasing-margin order subject to remaining budget $B_k$, producing $\bm N_k$
    \For{$i=1,\ldots,M$}
        \State If $N_{i,k}>0$, generate $N_{i,k}^{\mathrm{resp}}=gN_{i,k}$ responses from $\pi_{\phi^k}(\cdot\mid x_i)$ and observe $S_{i,k}$ successes and $F_{i,k}$ failures; otherwise, set $S_{i,k}=F_{i,k}=0$.
        \State Update $(a_{i,k+1},b_{i,k+1})$ by \eqref{eq:discounted-beta} and compute $\qhat_{i,k+1}$
        \State Compute $\zhat_{i,k+1}^*(\cbar\vartheta_i^k)$ by \eqref{eq:zstar-closed}
        \State Update $\vartheta_i^{k+1}$ by \eqref{eq:theta-mu-update-normalized}
    \EndFor
    \State Update $\nu^{k+1}$ by \eqref{eq:theta-mu-update-normalized} and set $B_{k+1}=B_k-\sum_iN_{i,k}$
    \State If $\sum_i N_{i,k}>0$, update the policy using all executed round-$k$ groups; otherwise, set $\phi^{k+1}=\phi^k$.
    \If{$B_{k+1}=0$}
        \State Terminate.
    \EndIf
\EndFor
\end{algorithmic}
\end{algorithm}


\section{Performance Guarantees}
We now analyze CERO with respect to the surrogate allocation objective in Section~\ref{sec:optimization}, while the downstream policy performance is evaluated empirically in Section~\ref{sec:experiments}.

\paragraph{Fixed-rate benchmark.}
Let $\mathbf N^{CERO}$ be the executed allocation and define
$$
\ALG_K^{\mathrm{stat}}(\bm q)
:=\sum_{i=1}^M U(q_i,z_{i,K}^{CERO}),
\qquad
\Reg_K^{\mathrm{stat}}
:=\OPT_K^{\mathrm{stat}}(\bm q)-\ALG_K^{\mathrm{stat}}(\bm q).
$$
Let $\bar N=\frac B {KM}=\frac R K$ denote the average per-prompt per-round budget. Let $\qhat_{i,k+1}$ be the post-feedback proxy used in round-$k$ prompt-specific dual update and define the average absolute proxy discrepancy
$$
\bar e_K^{\mathrm{stat}}
:=\frac{1}{MK}\sum_{k=1}^K\sum_{i=1}^M
|\qhat_{i,k+1}-q_i|.
$$

\begin{theorem}[Fixed-rate pathwise guarantee]
\label{thm:static}
Assume \cref{eq:capacity-assumption}. For any fixed target vector $\bm q\in[0,1/4]^M$ and any realized proxy sequence $\{\hat q_{i,k+1}\}$, Algorithm~\ref{alg:cero} with $\gamma_\vartheta=\gamma_\nu=\frac1{\sqrt K}$ satisfies
\begin{equation}
\frac{\Reg_K^{\mathrm{stat}}}{M}
\le
\min\left\{1,\ 
\lambda\frac{\Nmax+G_B/M}{\bar N}\frac1{\sqrt K}
+8\lambda\frac{\Nmax}{\bar N}\bar e_K^{\mathrm{stat}}
\right\}.
\label{eq:static-bound}
\end{equation}
\end{theorem}


This pathwise bound separates online optimization error from reward-proxy discrepancy.
In particular, if $N_{\max}/\bar N$, $G_B/(M\bar N)$, and $\lambda$ remain
bounded and $\bar e_K^{\mathrm{stat}}=O(K^{-1/2})$, then the regret per prompt
vanishes at rate $O(K^{-1/2})$. The proof is given in Appendix ~\ref{app:static-proof}.


\paragraph{Time-varying rates and a same-path benchmark.}
For a realized rate path $\bm q_{1:K}$, define the value of round-$k$ exposure using the current rate:
$$
\Delta U_{i,k}(\mathbf N;q_{i,k})
:=U(q_{i,k},z_{i,k})-U(q_{i,k},z_{i,k-1}).
$$
Using the same $q_{i,k}$ in both terms ensures that $\Delta U_{i,k}$ reflects only the utility gained from the new exposure $N_{i,k}$, rather than changes due to variation in the rate itself. Writing $\bm q_{1:K}=(\bm q_1,\ldots,\bm q_K)$, the same-path offline benchmark evaluates every feasible allocation plan against the same realized rate path:
$$
\OPT_K^{\mathrm{dyn}}(\bm q_{1:K})
:=\max_{\mathbf N\in\mathcal N_{B,C}}
\sum_{k=1}^K\sum_{i=1}^M\Delta U_{i,k}(\mathbf N;q_{i,k}).
$$
Define $\ALG_K^{\mathrm{dyn}}$ by evaluating the executed CERO allocation in the same way, and let $\Reg_K^{\mathrm{dyn}}=\OPT_K^{\mathrm{dyn}}-\ALG_K^{\mathrm{dyn}}$. By holding the realized rate path fixed across allocation plans, the benchmark focuses specifically on the quality of allocation decisions along a common realized rate path. Set
$$
\bar e_K^{\mathrm{dyn}}
:=\frac1{MK}\sum_{k=1}^K\sum_{i=1}^M
|\qhat_{i,k+1}-q_{i,k}|,~
V_K:=\sum_{k=1}^{K-1}\norm{\bm q_{k+1}-\bm q_k}_\infty.
$$

\begin{theorem}[Same-path drifting-rate guarantee]
\label{thm:dynamic}
Under the same algorithmic setup and step sizes as in \Cref{thm:static}, for every realized rate path $\bm q_{1:K}$ and proxy sequence $\{\qhat_{i,k+1}\}$, Algorithm~\ref{alg:cero} satisfies
\begin{equation}
\frac{\Reg_K^{\mathrm{dyn}}}{M}
\le
\min\left\{\lambda,\ 
\lambda\frac{\Nmax+G_B/M}{\bar N}\frac1{\sqrt K}
+8\lambda\frac{\Nmax}{\bar N}\bar e_K^{\mathrm{dyn}}
+4\lambda\left(1+\frac{\Nmax}{\bar N}\right)V_K
\right\}.
\label{eq:dynamic-bound}
\end{equation}
\end{theorem}

Comparing with \Cref{thm:static}, the additional $V_K$ term explicitly isolates the effect of temporal drift from reward-proxy discrepancy. The proof is given in Appendix ~\ref{app:dynamic-proof}.

\section{Experiments}
\label{sec:experiments}
\subsection{Main Results}
\paragraph{Experimental setup.}
We evaluate DeepSeek-R1-Distill-Qwen-1.5B \citep{guo2025deepseek}, Qwen3-4B-Base \citep{yang2025qwen3}, and Qwen2.5-7B-Instruct \citep{qwen2025qwen25technicalreport}, training all methods on the same uniformly sampled 50\% subset of DAPO-Math-17K. Each run spans $K=500$ scheduler rounds with a training-response budget of $B^{\mathrm{resp}}=256{,}000$, and CERO uses eight responses per admitted prompt, with adaptation occurring through prompt admission and per-round group counts. We evaluate on AIME, AMC, MATH, MINERVA, and OlympiadBench using 16 responses per problem. We report avg@16 as the mean binary reward across responses and problems within each benchmark, and macro-average the five benchmark scores. Implementation details and compute measurements are provided in Appendices~\ref{sec:implem} and~\ref{app:computational_cost}, respectively.

\paragraph{Baselines.}
We evaluate four methods:
(i) \textbf{GRPO} \citep{guo2025deepseek}, a widely used reinforcement learning method for training reasoning models;
(ii) \textbf{KnapsackRL} \citep{li2025knapsack}, a representative prior method for adaptive rollout allocation, which dynamically allocates different numbers of rollouts across prompts;
(iii) \textbf{CERO}, our proposed method;
and (iv) \textbf{Random-matched CERO}, an ablation that preserves the rollout allocation produced by CERO while assigning it to a random permutation of prompts. The latter controls for the allocation pattern and isolates the effect of our prompt selection mechanism.

Table~\ref{tab:main} shows that CERO attains avg@16 macro-averages of 32.81, 45.36, and 41.81, improving over GRPO by 2.62, 2.10, and 1.59 points, respectively. It also achieves a higher macro-average than the adaptive KnapsackRL baseline on all three backbones, with the clearest improvement on DSR1-Q1.5B and smaller margins on the larger models. The Random-matched control further isolates the contribution of prompt-aware assignment. Its macro-average falls below CERO by 3.63, 1.50, and 1.63 points across the three backbones. Since Random-matched CERO preserves allocation sizes and timing while randomizing prompt assignment, this comparison supports the contribution of prompt-aware assignment under the same rollout schedule.

\begin{table*}[t]
    \centering
    \caption{
        Evaluation performance (avg@16) comparison across different models and benchmarks.
    }
    \label{tab:main}
    
    \small
    \setlength{\tabcolsep}{5pt}
    \renewcommand{\arraystretch}{1.12}

    \begin{tabular}{clcccccc}
        \toprule
        \textbf{Base Model}
        & \textbf{Method}
        & \textbf{AIME}
        & \textbf{AMC}
        & \textbf{MATH}
        & \textbf{MINERVA}
        & \textbf{Olympiad}
        & \textbf{Avg.} \\
        \midrule

        \multirow{4}{*}{\makecell[c]{DSR1-Q1.5B}}
        & GRPO
        & 5.62
        & 32.23
        & 66.12
        & 21.90
        & 25.07
        & 30.19 \\

        & KnapsackRL
        & 4.58
        & 32.38
        & 65.14
        & 20.93
        & 24.59
        & 29.53 \\

        & \ourscell{\textbf{CERO (ours)}}
        & \ourscell{\textbf{7.92}}
        & \ourscell{\textbf{38.03}}
        & \ourscell{\textbf{68.05}}
        & \ourscell{\textbf{22.66}}
        & \ourscell{\textbf{27.41}}
        & \ourscell{\textbf{32.81}} \\

        & \makecell[l]{Random-matched CERO}
        & 4.17
        & 31.02
        & 64.54
        & 21.48
        & 24.69
        & 29.18 \\

        \midrule

        \multirow{4}{*}{Qwen3-4B}
        & GRPO
        & 13.54
        & 50.98
        & 80.10
        & 30.72
        & 40.94
        & 43.26 \\

        & KnapsackRL
        & \textbf{17.50}
        & 53.24
        & 80.08
        & 31.36
        & 41.26
        & 44.69 \\

        & \ourscell{\textbf{CERO (ours)}}
        & \ourscell{17.29}
        & \ourscell{\textbf{53.69}}
        & \ourscell{\textbf{80.83}}
        & \ourscell{31.94}
        & \ourscell{\textbf{43.06}}
        & \ourscell{\textbf{45.36}} \\

        & \makecell[l]{Random-matched CERO}
        & \textbf{17.50}
        & 50.00
        & 79.13
        & \textbf{32.84}
        & 39.82
        & 43.86 \\

        \midrule
        \multirow{4}{*}{Q2.5-7B-Inst}
        & GRPO
        & 12.92
        & 48.42
        & 74.36
        & 28.06
        & 37.34
        & 40.22 \\

        & KnapsackRL
        & 13.33
        & \textbf{50.90}
        & 77.15
        & 28.72
        & \textbf{38.64}
        & 41.75 \\

        & \ourscell{\textbf{CERO (ours)}}
        & \ourscell{\textbf{15.83}}
        & \ourscell{49.77}
        & \ourscell{\textbf{77.44}}
        & \ourscell{27.55}
        & \ourscell{38.45}
        & \ourscell{\textbf{41.81}} \\

        & \makecell[l]{Random-matched CERO}
        & 11.88
        & 46.23
        & 76.35
        & \textbf{29.30}
        & 37.17
        & 40.18 \\

        \bottomrule
    \end{tabular}
\end{table*}

\subsection{Ablation and Mechanism Analysis}
\paragraph{Is the reward-variation proxy informative?}
We examine whether the proxy anticipates the reward contrast available in future rollout groups. On the DSR1-Q1.5B backbone, we pair each pre-decision prediction $\hat q_{i,k}$ with the unbiased sample variance of the eight binary rewards generated afterward, covering $32,000$ selected prompt--round groups. Across ten equal-budget intervals, higher predicted means tend to coincide with greater observed reward variation (Pearson $r=0.769$; Figure~\ref{fig:qhat-tracking}), showing that the proxy tracks changes in reward contrast during training. Across seven prediction bins, higher-prediction bins generally exhibit larger observed mean variance, with closely preserved relative ordering (Pearson $r=0.971$, Spearman $\rho=0.964$; Figure~\ref{fig:qhat_calibration_reliability}). Together, these analyses provide evidence that the proxy offers an informative pre-decision signal for subsequent reward contrast among selected prompt--round groups. Appendix~\ref{app:qhat_reliability} further examines absolute calibration.
\begin{figure}[t]
    \centering

    \begin{minipage}[t]{0.48\linewidth}
        \centering
        \includegraphics[width=\linewidth]{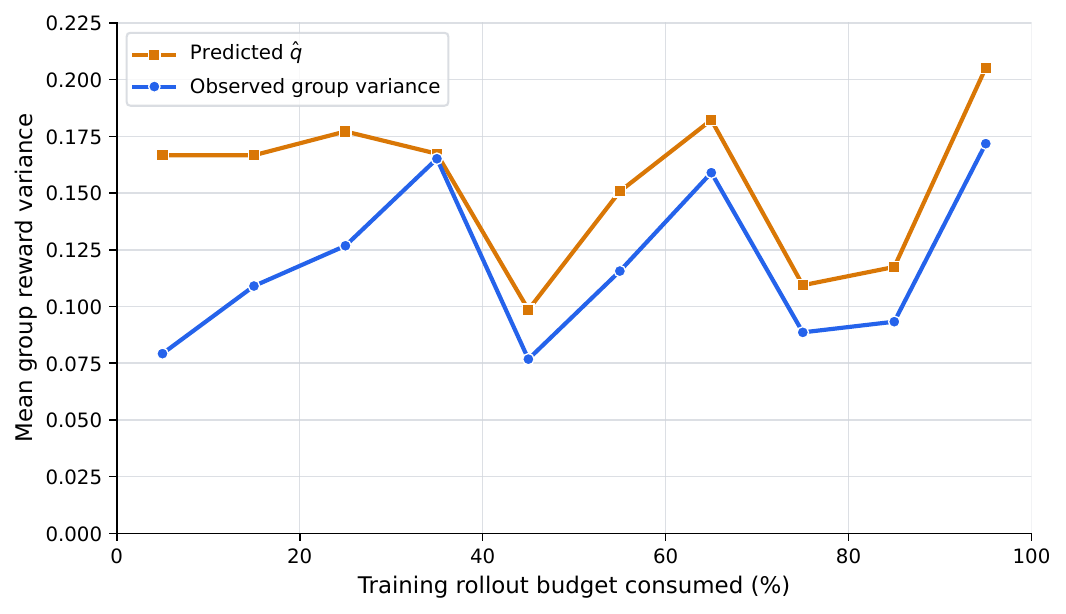}
        \setlength{\abovecaptionskip}{1pt}
        \caption{Temporal tracking of CERO's reward-variation proxy on DSR1-Q1.5B.
        }
        \label{fig:qhat-tracking}
    \end{minipage}
    \hfill
    \begin{minipage}[t]{0.48\linewidth}
        \centering
        \includegraphics[width=\linewidth]{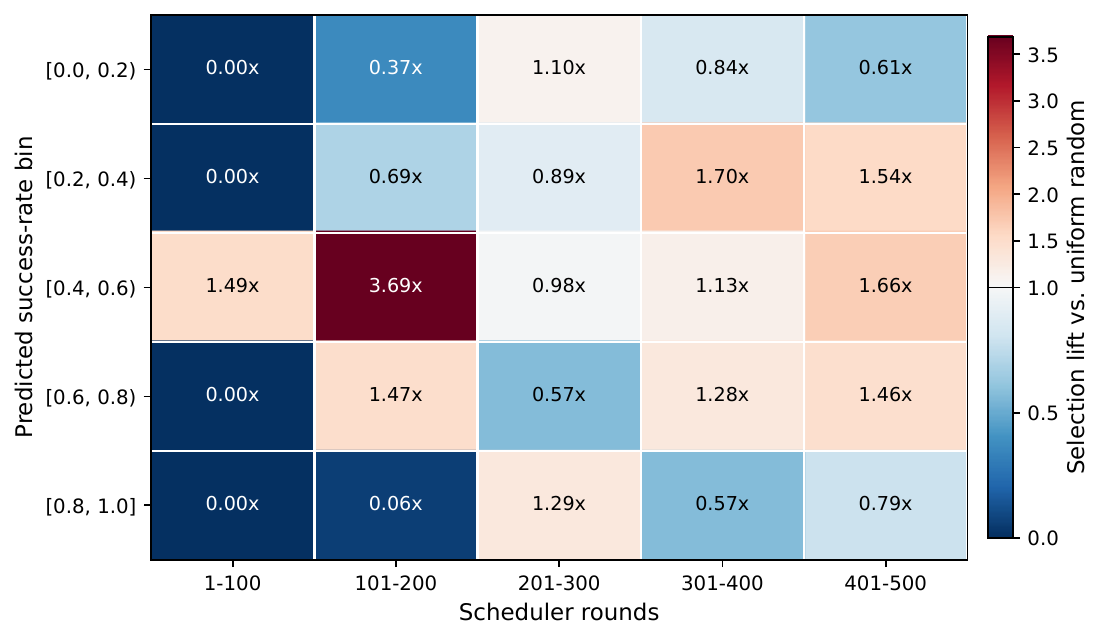}
        \setlength{\abovecaptionskip}{1pt}
        \caption{Prompt-selection lift by pre-decision predicted success rate.
        }
        \label{fig:heat}
    \end{minipage}
\end{figure}

\paragraph{Does CERO prioritize informative prompts?}
We characterize CERO's selection pattern by partitioning the 500 scheduler rounds into five 100-round stages and binning candidate prompt–round pairs by their pre-decision predicted success rate. For bin $b$ in stage $t$, we define the selection lift as
\[
L_{b,t}
=
\frac{n^{\mathrm{sel}}_{b,t}/n^{\mathrm{cand}}_{b,t}}
     {n^{\mathrm{sel}}_{t}/n^{\mathrm{cand}}_{t}},
\]
where $n^{\mathrm{sel}}$ and $n^{\mathrm{cand}}$ denote the numbers of selected and candidate prompt--round pairs. Consequently, $L_{b,t}=1$ corresponds to uniform random selection, while $L_{b,t}>1$ and $L_{b,t}<1$ indicate over-selection and under-selection respectively. 
\Cref{fig:heat} shows that CERO preferentially selects prompts with intermediate predicted success rates. Over the full run, the $[0.4,0.6)$ bin has a $1.72\times$ selection lift, while both extreme bins have approximately $0.65\times$. With the full prompt pool available in every round, the lift directly measures CERO's selection preference across predicted-success ranges relative to uniform selection. This pattern aligns with CERO's reward-variation-based design: prompts for which both success and failure remain plausible offer opportunities for the within-group reward contrast needed by group-relative learning. The strength and breadth of the preference vary across training stages, with all three interior bins over-selected in the final stage. 

\paragraph{Does prompt-aware admission increase the fraction of effective groups?}
Let $S_{i,k}$ of $G_{i,k}$ responses in group $(i,k)$ be correct. A group is effective if $0<S_{i,k}<G_{i,k}$. Define the following effective-group rate
\[
\mathrm{EGR}
:=
\frac{\sum_{i,k}\mathbf{1}\{0<S_{i,k}<G_{i,k}\}}
{\sum_{i,k}\mathbf{1}\{G_{i,k}>0\}},
\]
the fraction of executed groups supplying nonzero reward contrast.
\Cref{fig:effective} shows that under the matched budget of $256,000$ training responses, CERO produces $18,347$ effective groups out of $32,000$ executed groups, yielding an EGR of 57.3\%, compared with 52.7\% for Random-matched CERO and 48.7\% for KnapsackRL. Because Random-matched CERO preserves CERO's group sizes and spending schedule while randomizing prompt assignment, the 4.6-point gap provides controlled evidence that prompt-aware admission increases the fraction of mixed-outcome groups under the same rollout schedule.


\begin{figure}[t]
    \centering

    \begin{minipage}[t]{0.48\linewidth}
        \centering
        \includegraphics[width=\linewidth]{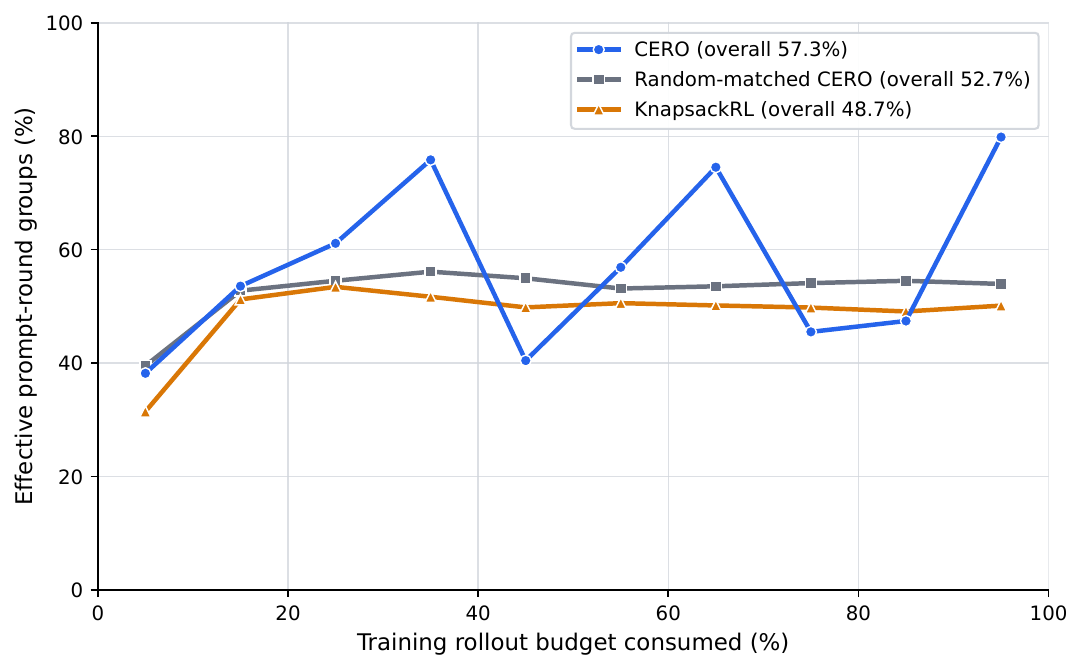}
        \setlength{\abovecaptionskip}{0.1pt}
        \caption{
        Effective-group rate over consumed rollout budget. 
        }
        \label{fig:effective}
    \end{minipage}
    \hfill
    \begin{minipage}[t]{0.48\linewidth}
        \centering
        \includegraphics[width=\linewidth]{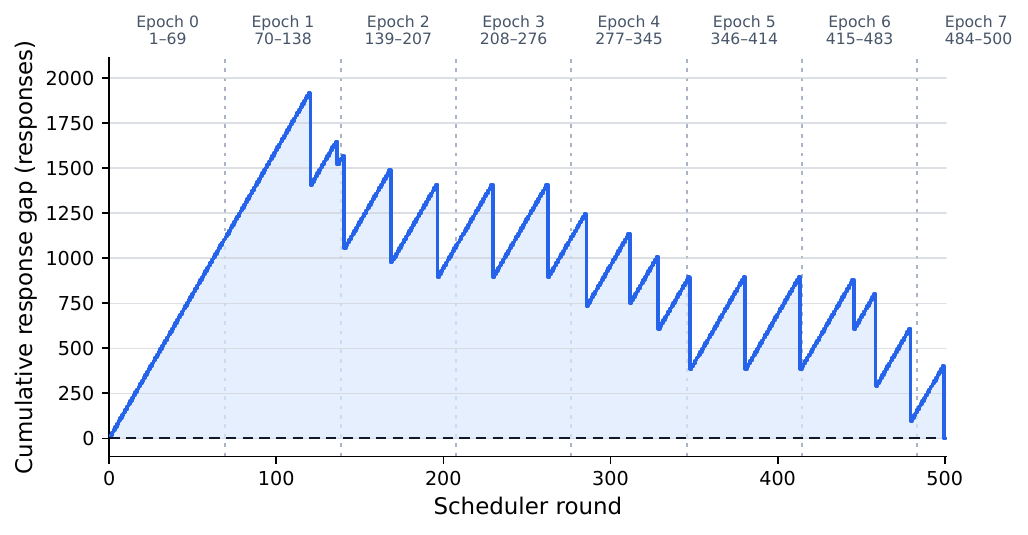}
        \setlength{\abovecaptionskip}{1pt}
        \caption{Cross-epoch budget pacing in CERO. Epoch boundaries are marked above trajectory. 
        }
        \label{fig:budget}
    \end{minipage}
\end{figure}

\paragraph{What do dual-based selection and adaptive pacing contribute?}
We first characterize CERO's budget pacing through the cumulative response gap from a uniform spending schedule:
\[
G_k^{\mathrm{exec}}
:=
\sum_{\tau=1}^{k}\sum_iN_{i,\tau}^{\mathrm{resp}}
-k\bar B^{\mathrm{resp}},
\qquad
\bar B^{\mathrm{resp}}
:=
\frac{B^{\mathrm{resp}}}{K}=512.
\]
Positive and negative gaps indicate spending ahead of and behind the uniform schedule, respectively.
\Cref{fig:budget} shows that spending deviations persist across data-epoch boundaries. The cumulative surplus peaks at 1,920 responses and returns to zero at the end of training, exhausting the 256,000-response budget exactly. This trajectory characterizes CERO’s nonuniform budget consumption.

\Cref{tab:ablation} evaluates the contributions of selection and pacing. CERO-fixed retains CERO's reward-proxy and prompt-dual updates but fixes $\nu=0$ and allocates $512$ responses per round. CERO-preset follows a predetermined front-loaded schedule: $528$ responses per round for the first $250$ rounds and $496$ for the remaining $250$ rounds. CERO-greedy retains CERO-fixed's proxy estimator, utility function, and per-round quota, but selects prompts by the marginal utility of an additional complete group rather than by prompt-dual scores.

\begin{wraptable}[11]{r}{0.36\textwidth}
\vspace{-1cm}
    \centering
    \caption{Prompt-selection and pacing ablations on DSR1-Q1.5B. Avg. denotes the five-benchmark avg@16 macro-average (mean $\pm$ std over 3 seeds).}
    \label{tab:ablation}
    \begin{tabular}{cc}
        \toprule
        \textbf{Method} & \textbf{Avg.} \\
        \midrule
        CERO & \textbf{32.26 $\pm$ 0.94} \\
        CERO-fixed & 31.13 $\pm$ 0.72 \\
        CERO-greedy & 30.97 $\pm$ 0.38 \\
        CERO-preset & 30.71 $\pm$ 0.54 \\
        \bottomrule
    \end{tabular}
\end{wraptable}
Across three training seeds, CERO achieves the highest macro-average of \(32.26\pm0.94\), outperforming CERO-fixed, CERO-preset, and CERO-greedy by $1.13$, $1.55$, and $1.29$ points on average, respectively. The gains over CERO-fixed and CERO-preset support the value of adapting budget consumption over the training horizon rather than using a fixed or predetermined spending schedule. 
Under the fixed 512-response quota, dual-guided and direct marginal-utility ranking yield similar mean scores of 31.13 and 30.97, whereas full CERO improves on both by more than one point. Together, these results support CERO's central design: jointly coordinating prompt prioritization and budget pacing over the training horizon.

\section{Conclusion}
We introduced CERO, a finite-horizon scheduler that coordinates prompt admission, repeated selection, and budget pacing under a shared rollout budget. The Fenchel-based primal--dual formulation yields tractable round-level decisions and pathwise guarantees for the surrogate allocation objective under fixed-rate and same-path time-varying benchmarks. Under matched training-response budgets, CERO achieves the highest avg@16 macro-average on all three backbones. Mechanistic analyses link its prompt choices to stronger within-group supervision, while multi-seed ablations show gains from adaptive pacing over fixed and preset schedules. Together, these results demonstrate the practical value of horizon-level rollout coordination for efficient RL post-training.

\bibliographystyle{plainnat}
\bibliography{bibliography}

\clearpage
\begin{APPENDICES}
\crefalias{section}{appendix}
\renewcommand{\theHsection}{appendix.\Alph{section}}

\section{Related Work}
\label{app:lr}
\textbf{Difficulty-aware sampling and filtering.}
Recent work has explored improving RL efficiency by prioritizing training examples and rollouts according to their difficulty or informativeness. At the query level, online curriculum and sample-selection methods dynamically adjust the training distribution to favor examples with greater learning value. Goldilocks RL \citep{mahrooghi2026goldilocks} emphasizes problems of intermediate difficulty, while SPEED-RL \citep{zhang2025speed}, SEC \citep{chen2025self}, and VADE \citep{hu2025vade} similarly use online estimates of difficulty or training utility to prioritize informative samples. Related approaches operate after rollout generation by filtering or pruning trajectories with limited contribution to policy updates. CPPO \citep{lin2025cppo} removes completions with low learning signals, PODS \citep{xu2025not} selects informative subsets of generated responses, and GFPO \citep{shrivastava2025sample} filters rollouts based on their utility relative to generation cost. These methods collectively reduce computation spent on overly easy, overly difficult, redundant, or otherwise uninformative samples.

\paragraph{Adaptive rollout allocation.}
Recent studies have increasingly treated rollout generation as a resource allocation problem, replacing the fixed group size in standard GRPO \citep{shao2024deepseekmath} with query-dependent rollout budgets. Some methods adapt allocation according to estimated difficulty, assigning more rollouts to challenging or uncertain queries and fewer to those with saturated outcomes \citep{liao-etal-2025-enhancing,zhang2025improving}. Others formulate allocation through explicit estimates of training value, uncertainty, or marginal utility, including Knapsack RL \citep{li2025knapsack}, CoBA-RL \citep{yao2026coba} and VIP \citep{nguyen2026adaptive}. A related direction makes allocation decisions progressively from intermediate rollout outcomes: Pilot-Commit \citep{kim2026spend}, AERO \citep{zhang2026train}, and VIGOR \citep{jiang2026learning} use early samples to determine whether additional generation is worthwhile. Collectively, these methods improve rollout efficiency by directing computation toward queries where additional exploration is expected to provide the greatest learning value.

\paragraph{Online resource allocation.}
Online resource allocation is a central topic in the online optimization literature and has been applied in many operational settings, including ad assignment \citep{mehta2007adwords}, revenue management \citep{talluri2006theory,jasin2012re, jiang2025degeneracy}, and online packing/covering problems \citep{buchbinder2005online,buchbinder2006improved,feldman2010online}. The existing algorithmic literature can be broadly organized according to how it exploits dual information. A dominant class of methods uses dual prices as signals for making online allocation decisions. Within this class, some papers adopt a learn-once strategy: they reserve an initial set of arrivals for estimation, compute a price vector from that sample, and then use the same prices for all future decisions \citep{devanur2009adwords,feldman2010online,molinaro2014geometry,devanur2012online}. Other papers advocate adaptive updating, where the pricing rule is revised at multiple points over time through repeated re-solving, leading to stronger robustness and improved horizon dependence \citep{agrawal2014dynamic,li2022online,chen2015dynamic}. In parallel, other works depart from this paradigm by developing primal-oriented algorithms \citep{kesselheim2014primal} or first-order schemes that do not require repeated computation of dual prices \citep{agrawal2014fast,balseiro2020dual}.

\section{Additional Results and Diagnostics}
\label{app:additional_results}

\subsection{Pass@16 Evaluation}
\label{app:best16_eval}

In addition to the avg@16 results reported in the main text, we evaluate pass@16 performance to measure success coverage under repeated sampling. Pass@16 is the percentage of problems with at least one correct response among these 16 samples:

$$
\mathrm{pass}@16
=
\frac{100}{|\mathcal D|}
\sum_{x\in\mathcal D}
\mathbf{1}\left\{
\sum_{j=1}^{16}r_{x,j}>0
\right\}.
$$

For binary rewards, this is equivalent to averaging the maximum reward among the 16 responses for each problem. 
Table~\ref{tab:best16} shows that CERO improves the pass@16 macro-average over GRPO on all three backbones, with gains of $1.71$, $0.39$, and $1.58$ points. It achieves the highest macro-average on DeepSeek-R1-Distill-Qwen-1.5B and Qwen2.5-7B-Instruct, reaching $50.99$ and $62.34$, respectively. On Qwen3-4B, KnapsackRL achieves the highest score of $67.09$, compared with $65.45$ for CERO. Together with the avg@16 results, these findings show improvements over GRPO in both average response accuracy and repeated-sampling success coverage.

\begin{table*}[t]
    \centering
    \caption{
        Evaluation performance (pass@16) comparison across different models and benchmarks.
    }
    \label{tab:best16}
    
    \small
\setlength{\tabcolsep}{5pt}
\renewcommand{\arraystretch}{1.12}

\begin{tabular}{clcccccc}
    \toprule
    \textbf{Base Model}
    & \textbf{Method}
    & \textbf{AIME}
    & \textbf{AMC}
    & \textbf{MATH}
    & \textbf{MINERVA}
    & \textbf{Olympiad}
    & \textbf{Avg.} \\
    \midrule

    \multirow{4}{*}{\makecell[c]{DSR1-Q1.5B}}
    & GRPO
    & \textbf{20.00}
    & 55.42
    & 83.40
    & \textbf{46.69}
    & 40.89
    & 49.28 \\

    & KnapsackRL
    & \textbf{20.00}
    & 56.63
    & 82.80
    & 42.65
    & 40.59
    & 48.53 \\

    & \ourscell{\textbf{CERO (ours)}}
    & \ourscell{\textbf{20.00}}
    & \ourscell{\textbf{62.65}}
    & \ourscell{\textbf{84.20}}
    & \ourscell{44.85}
    & \ourscell{\textbf{43.26}}
    & \ourscell{\textbf{50.99}} \\

    & \makecell[l]{Random-matched CERO}
    & 13.33
    & 51.81
    & 81.40
    & 45.96
    & 39.70
    & 46.44 \\

    \midrule

    \multirow{4}{*}{Qwen3-4B}
    & GRPO
    & 40.00
    & 79.52
    & 92.00
    & 51.10
    & 62.67
    & 65.06 \\

    & KnapsackRL
    & \textbf{46.67}
    & 80.72
    & \textbf{92.20}
    & 51.84
    & \textbf{64.00}
    & \textbf{67.09} \\

    & \ourscell{\textbf{CERO (ours)}}
    & \ourscell{40.00}
    & \ourscell{\textbf{83.13}}
    & \ourscell{\textbf{92.20}}
    & \ourscell{\textbf{52.21}}
    & \ourscell{59.70}
    & \ourscell{65.45} \\

    & \makecell[l]{Random-matched CERO}
    & 43.33
    & 73.49
    & 92.00
    & 51.10
    & 61.04
    & 64.19 \\

    \midrule
    \multirow{4}{*}{Q2.5-7B-Inst}
    & GRPO
    & 33.33
    & 73.49
    & 91.40
    & 47.79
    & 57.78
    & 60.76 \\

    & KnapsackRL
    & 30.00
    & 74.70
    & \textbf{92.00}
    & \textbf{48.53}
    & \textbf{60.30}
    & 61.10 \\

    & \ourscell{\textbf{CERO (ours)}}
    & \ourscell{\textbf{36.67}}
    & \ourscell{\textbf{78.31}}
    & \ourscell{89.40}
    & \ourscell{\textbf{48.53}}
    & \ourscell{58.81}
    & \ourscell{\textbf{62.34}} \\

    & \makecell[l]{Random-matched CERO}
    & 30.00
    & 71.08
    & 88.80
    & 46.69
    & 56.74
    & 58.66 \\

    \bottomrule
\end{tabular}
\end{table*}

\subsection{Reliability of the Reward-Variation Proxy}
\label{app:qhat_reliability}

We examine whether CERO's pre-decision reward-variation proxy captures differences in subsequently observed group variance, both across prediction levels and over the training trajectory.

\paragraph{Diagnostic setup.}
We analyze 32,000 executed prompt--round groups from the DSR1-Q1.5B run. For each group, we pair the pre-decision prediction $\hat q_{i,k}$ with the unbiased sample variance $v_{i,k}=\widehat{\operatorname{Var}}_8(r_{i,k})$ of the eight binary rewards generated afterward. Figure~\ref{fig:qhat_calibration_reliability} aggregates these pairs into seven nonempty prediction bins and plots the mean prediction $\bar q_\ell$ against the mean observed variance $\bar v_\ell$. Vertical error bars show nominal 95\% normal-approximation intervals,
\[
\bar v_\ell \pm 1.96\,\frac{s_\ell}{\sqrt{n_\ell}},
\]
where $n_\ell$ is the number of groups in bin $\ell$ and $s_\ell$ is the sample standard deviation of their observed group variances. These intervals use group-level standard errors within a single run, without adjustment for repeated-prompt or temporal dependence.

\paragraph{Association across prediction levels.}
Figure~\ref{fig:qhat_calibration_reliability} shows a strong positive association between predicted and observed bin means, with Pearson correlation $r=0.971$ and Spearman correlation $\rho=0.964$. Across the plotted prediction range, the observed mean group variance rises overall from approximately 0.03 to 0.20. Higher-prediction bins generally correspond to greater realized reward variation, and the ordering of the bin means is closely preserved despite a local fluctuation in the intermediate range. The proxy therefore distinguishes prediction ranges associated with substantially different levels of subsequent reward contrast. Because the predictions are recorded before the current-round responses are generated, this association reflects information available when CERO makes its allocation decisions.

\paragraph{Absolute calibration.}
The predicted values also exhibit an upward offset in numerical scale. Across the selected groups, the mean predicted variance is 0.154, compared with a mean observed variance of 0.119. Consistent with this difference, the bin means in Figure~\ref{fig:qhat_calibration_reliability} lie below the diagonal marking equality between predictions and observations. The strong correlations and the displacement from the diagonal characterize two distinct properties: the proxy captures differences across prediction levels while tending to overestimate absolute reward variance. This distinction matters for CERO because $\hat q_{i,k}$ enters the surrogate utility and prompt-dual updates numerically, rather than serving only as a ranking statistic. The diagnostic thus identifies both the information carried by the proxy and the calibration offset present in the analyzed run.

\paragraph{Temporal tracking and allocation relevance.}
Figure~\ref{fig:qhat-tracking} provides a complementary view by aggregating the same trajectory into ten consecutive equal-budget intervals of 3,200 groups each. Predicted and observed interval means co-vary over training, with Pearson correlation $r=0.769$. Their trajectories capture common decreases and rebounds, showing that the proxy reflects changes in reward variation as training progresses, in addition to the differences across prediction levels shown in Figure~\ref{fig:qhat_calibration_reliability}. This behavior is relevant to group-relative supervision: mixed-outcome groups provide positive reward variance, whereas all-correct and all-incorrect groups provide no within-group reward contrast. Together, the prediction-bin and temporal analyses support $\hat q_{i,k}$ as an informative prospective signal for CERO's allocation mechanism. They complement the selection-lift and effective-group analyses in Section~6.2, which examine how the resulting allocation translates that signal into prompt choices and usable supervision.

\begin{figure}[t]
    \centering
    \includegraphics[
        width=0.8\linewidth
    ]{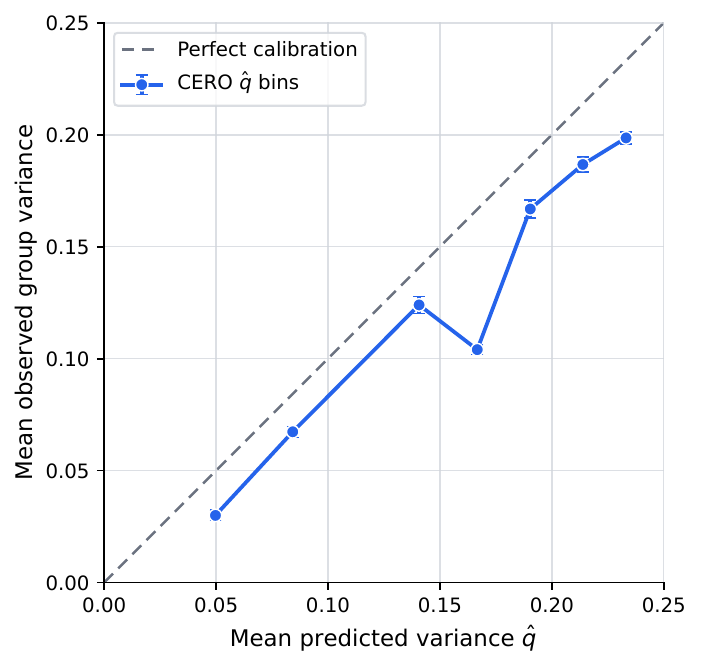}
    \caption{Calibration of CERO's reward-variation proxy on DSR1-Q1.5B. Points show predicted and observed bin means across seven prediction bins; vertical bars indicate nominal 95\% within-run intervals for the observed means. The dashed diagonal marks perfect calibration.}
    \label{fig:qhat_calibration_reliability}
\end{figure}

\subsection{Cross-Round Prompt Exposure}
To examine how CERO redistributes repeated exposure over the training horizon, we compare terminal cumulative prompt exposure with GRPO. Following Section~\ref{sec:optimization}, we normalize each prompt's terminal exposure by the equal-share scale \(R=B/M\). In Figure~\ref{fig:exposure_cdf}, the empirical CDF at \(x\) gives the fraction of prompts with \(z_{i,K}/R\leq x\). Because both methods use fixed-size response groups and the same total rollout budget, differences in terminal exposure reflect how many times each prompt is selected over the training horizon rather than differences in group size or aggregate exposure.
\label{app:exposure}
\begin{figure*}[t]
    \centering
    \includegraphics[
        width=\textwidth
    ]{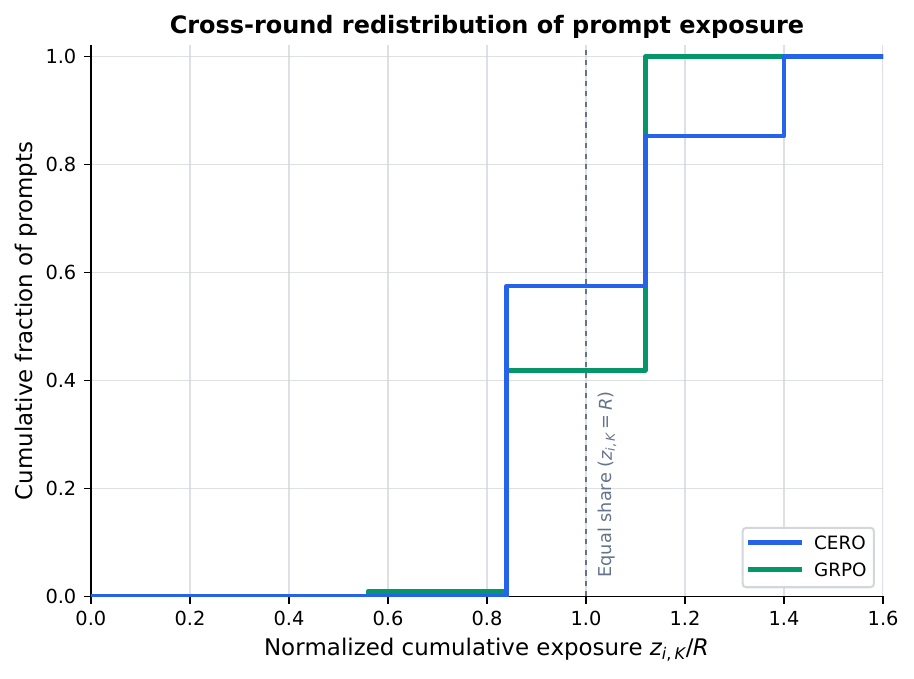}
    \caption{Empirical CDF of normalized terminal prompt exposure \(z_{i,K}/R\) for CERO and GRPO on DSR1-Q1.5B. The dashed line marks the equal-share exposure $R=B/M$.}
    \label{fig:exposure_cdf}
\end{figure*}

Figure~\ref{fig:exposure_cdf} shows that CERO induces a broader terminal-exposure distribution than GRPO. CERO spreads terminal exposure across three levels, with 57.5\% of prompts at \(0.84R\), 27.8\% at \(1.12R\), and a 14.7\% high-exposure tail at \(1.40R\). GRPO is more concentrated, with 41.0\% and 58.1\% of prompts at \(0.84R\) and \(1.12R\), respectively. Thus, CERO shifts exposure away from the more concentrated GRPO pattern: a larger fraction of prompts receive below-equal-share exposure, while a smaller subset is revisited sufficiently often to reach substantially above the equal-share level. Since both methods use the same total rollout budget, this broader distribution reflects a redistribution of repeated selections across prompts rather than additional aggregate exposure. In particular, CERO trades fewer revisits to some prompts for more persistent revisiting of others over the training horizon. 

\subsection{Computational Cost}
\label{app:computational_cost}

We compare the training cost of CERO, Random-matched CERO, KnapsackRL, and
vanilla GRPO. For every backbone, all methods use the same number of GPUs,
500 scheduler rounds, and a matched budget of 256,000 training responses.
End-to-end GPU-hours are calculated as the Slurm elapsed time multiplied by
the number of allocated GPUs. This measurement includes initialization,
in-training validation, and checkpointing.

Figure~\ref{fig:compute_time_comparison} provides a detailed comparison for
the 1.5B backbone. In addition to end-to-end GPU-hours, we report core
training time normalized by the number of generated response tokens. The
normalized measurement excludes validation and checkpointing but includes
rollout generation, reward computation, log-probability computation, policy
optimization, and allocation-related operations.

\begin{figure*}[t]
    \centering
    \includegraphics[
        width=\textwidth
    ]{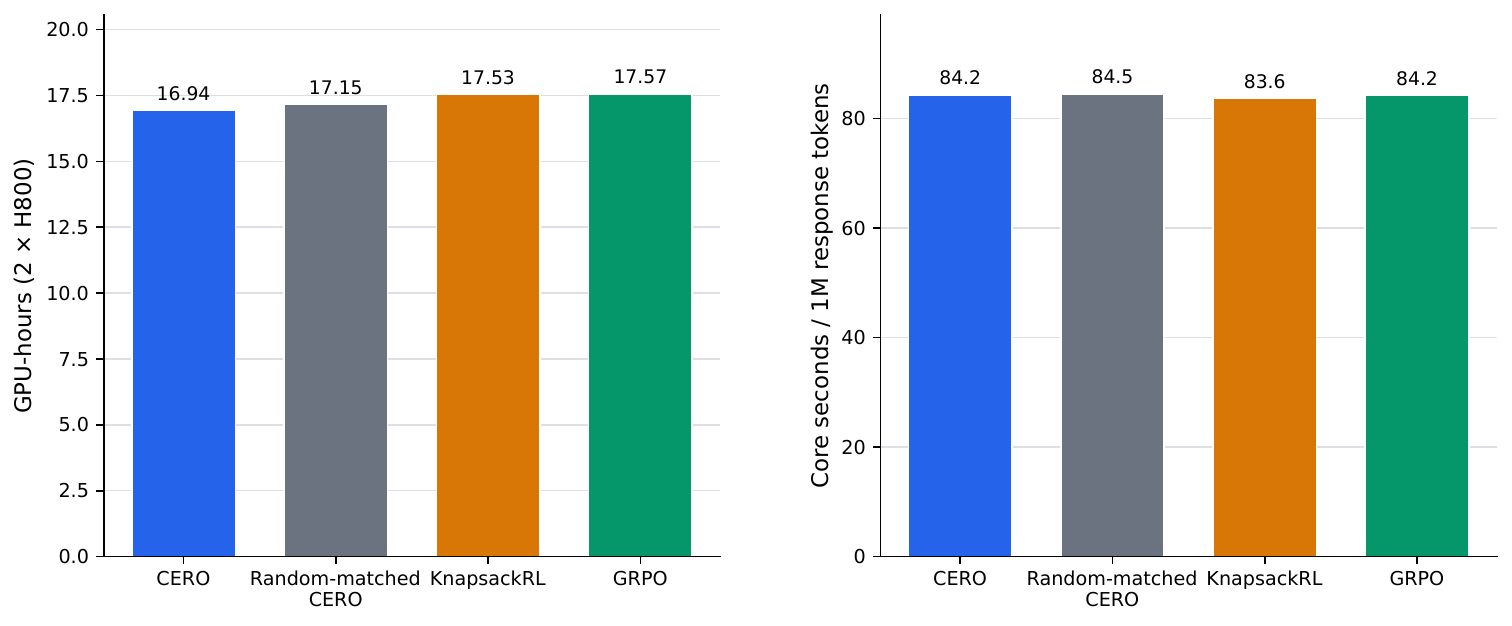}
    \caption{Training-compute comparison on DeepSeek-R1-Distill-Qwen-1.5B under a matched response budget. The left panel reports end-to-end GPU-hours, while the right panel reports core training time normalized by generated response tokens.}
    \label{fig:compute_time_comparison}
\end{figure*}

On the 1.5B backbone, CERO, Random-matched CERO, KnapsackRL, and vanilla GRPO
consume $16.94$, $17.15$, $17.53$, and $17.57$ GPU-hours, respectively.
Their token-normalized core costs are $84.2$, $84.5$, $83.6$, and $84.2$
seconds per million generated response tokens. The normalized measurements
differ by only approximately $1.1\%$ of their mean value, with CERO being
nearly identical to vanilla GRPO. CERO therefore combines adaptive allocation with token-normalized core training cost comparable to the baselines in the 1.5B experiment.

We also measure end-to-end GPU-hours for the larger backbones. On
Qwen3-4B-Base, CERO uses $36.31$ GPU-hours, Random-matched CERO uses $36.78$
GPU-hours, KnapsackRL uses $35.85$ GPU-hours, and vanilla GRPO uses $36.89$
GPU-hours. These values correspond to wall-clock times of $9.08$, $9.19$,
$8.96$, and $9.22$ hours, respectively, using four H800 GPUs. The differences
among the four methods are within approximately $3\%$ of their mean compute
cost.

On Qwen2.5-7B-Instruct, CERO, Random-matched CERO, KnapsackRL, and vanilla
GRPO use $31.46$, $32.78$, $34.99$, and $31.72$ GPU-hours, respectively.
Their corresponding wall-clock times are $7.86$, $8.19$, $8.75$, and $7.93$
hours on four H800 GPUs. CERO therefore remains comparable to the baselines
in end-to-end computational cost as the backbone size increases.

The end-to-end measurements capture the practical cost of each complete training protocol under the matched response budget, including differences in response lengths and update counts. CERO and Random-matched CERO perform 488, 486, and 487 policy updates on the 1.5B, 4B, and 7B backbones, respectively, while GRPO and KnapsackRL perform 500. Zero-allocation rounds account for the lower update counts.

Each cost measurement is obtained from one completed run per method and backbone. Across the reported runs, CERO combines adaptive prompt selection and cross-epoch budget pacing with comparable end-to-end training cost. On the 1.5B backbone, its token-normalized core cost matches GRPO at 84.2 seconds per million generated response tokens.

\section{Utility Properties and Full-Information Oracle}
\subsection{Exact Conjugate, Fenchel Representation, and Stability}
\label{app:fenchel}
Fix $q\in[0,1/4]$, set $c=4\lambda q/R$, and write $U_c(z)=1-e^{-cz}$ on $[0,Z]$. Define
\begin{equation}
U_{c,Z}^*(\theta):=\inf_{0\le z\le Z}\{\theta z-U_c(z)\}.
\end{equation}

\begin{lemma}[Closed-form compact conjugate]
\label{lem:conjugate}
If $c=0$, then $U_{c,Z}^*(\theta)=0$ for every $\theta\ge0$. If $c>0$, then
\begin{equation}
U_{c,Z}^*(\theta)=
\begin{cases}
\theta Z-1+e^{-cZ}, & 0\le\theta\le ce^{-cZ},\\[1mm]
\dfrac{\theta}{c}\log\dfrac{c}{\theta}+\dfrac{\theta}{c}-1,
& ce^{-cZ}<\theta<c,\\[2mm]
0, & \theta\ge c.
\end{cases}
\label{eq:conjugate-closed}
\end{equation}
The minimizer is given by \eqref{eq:zstar-closed}.
\end{lemma}

\begin{proof}
For $c>0$, differentiate $\theta z-1+e^{-cz}$. Its derivative is $\theta-ce^{-cz}$ and is strictly increasing in $z$. If $\theta\le ce^{-cZ}$, the objective is nonincreasing on $[0,Z]$ and the minimum is at $Z$. If $ce^{-cZ}<\theta<c$, the unique root is $z=c^{-1}\log(c/\theta)$. If $\theta\ge c$, the minimum is at $0$. Substitution yields \eqref{eq:conjugate-closed}. The case $c=0$ is immediate.
\end{proof}

We also write $U_q^*(\theta):=U_{4\lambda q/R,Z}^*(\theta)$ for the compact conjugate parameterized by $q$; in particular, $U_i^*(\theta)=U_{q_i}^*(\theta)$ in Section~2.3.
\begin{lemma}[Exact compact biconjugate and stability]
\label{lem:fenchel-stability}
For every $q\in[0,1/4]$ and $z\in[0,Z]$,
\begin{equation}
U(q,z)=\min_{0\le\theta\le\cbar}\{\theta z-U_q^*(\theta)\}.
\end{equation}
Moreover, for $q,q'\in[0,1/4]$,
\begin{align}
\sup_{0\le z\le Z}|U(q,z)-U(q',z)|
&\le \frac{4\lambda Z}{R}|q-q'|,\\
\sup_{0\le\theta\le\cbar}|U_q^*(\theta)-U_{q'}^*(\theta)|
&\le \frac{4\lambda Z}{R}|q-q'|.
\label{eq:stability-appendix}
\end{align}
\end{lemma}

\begin{proof}
The slope of $U(q,\cdot)$ lies in $[0,\cbar]$. For a fixed $z$, choose the supporting slope $\theta_z=\partial_zU(q,z)$, with one-sided derivatives at the endpoints. Concavity gives
\begin{equation}
U(q,s)\le U(q,z)+\theta_z(s-z),
\end{equation}
so $U_q^*(\theta_z)=\theta_zz-U(q,z)$. The reverse inequality follows directly from the definition of the conjugate.

For stability,
\begin{equation}
|\partial_qU(q,z)|
=\frac{4\lambda z}{R}e^{-4\lambda qz/R}
\le\frac{4\lambda Z}{R}.
\end{equation}
The first bound follows from the mean-value theorem. Applying a common uniform perturbation before taking the infimum over $z\in[0,Z]$ gives the conjugate bound.
\end{proof}

\subsection{Full-Information Oracle}
\label{app:full-information-oracle}

We characterize the fixed-rate full-information benchmark when the target rates $q=(q_1,\ldots,q_M)$ are known. Since the objective depends only on terminal prompt exposures, the oracle can first determine the optimal exposure vector and then distribute these exposures across the $K$ rounds.

\paragraph{Terminal-exposure formulation.}
Writing $z_i=\sum_{k=1}^{K}N_{i,k}$ and $c_i=4\lambda q_i/R$, the fixed-rate benchmark can be expressed as
\begin{equation}
\mathrm{OPT}^{\mathrm{stat}}_K(q)
=
\max_{\substack{
z\in\mathbb{Z}_{\ge 0}^{M}\\
z_i\le Z,\ \sum_i z_i\le B
}}
\sum_{i=1}^{M}\left(1-e^{-c_i z_i}\right).
\label{eq:oracle-terminal}
\end{equation}
Every feasible rollout plan induces an exposure vector satisfying these constraints. Conversely, every such integer vector admits a feasible round-level realization, as shown below. For the continuous relaxation, the uniform assignment $N_{i,k}=z_i/K$ is feasible because $z_i\le KN_{\max}$ and $\sum_i z_i\le B\le KC$.

\paragraph{Continuous water-filling solution.}
Relaxing the integrality constraint gives a separable concave optimization problem. Let $\xi\ge0$ be the multiplier for the total budget. For $c_i>0$, an interior optimum satisfies
$c_i e^{-c_i z_i}=\xi$.
Consequently, for $\xi>0$, the optimal exposure has the clipped water-filling form
\begin{equation}
z_i^{\mathrm{cont}}(\xi)
=
\left[
\frac{1}{c_i}\log\frac{c_i}{\xi}
\right]_{[0,Z]},
\qquad c_i>0,
\label{eq:oracle-water-filling}
\end{equation}
where $[x]_{[0,Z]}=\min\{Z,\max\{0,x\}\}$.

Let $\mathcal I_+=\{i:c_i>0\}$. When $0<B<Z|\mathcal I_+|$, the budget is binding and $\xi$ can be chosen so that
$\sum_{i\in\mathcal I_+}z_i^{\mathrm{cont}}(\xi)=B$.
The left-hand side is continuous and nonincreasing in $\xi$, allowing the price to be found by one-dimensional bisection. Prompts with $c_i=0$ contribute zero utility and may be assigned zero exposure. When $B\ge Z|\mathcal I_+|$, assigning $Z$ units to every positive-rate prompt is optimal; any remaining budget may be left unused or assigned to zero-rate prompts without changing the objective.

\paragraph{Exact integer allocation.}
For integer exposures, the marginal utility of the $\ell$-th allocation unit assigned to prompt $i$ is
\begin{equation}
\Delta_i(\ell)
:=
U(q_i,\ell)-U(q_i,\ell-1)
=
(1-e^{-c_i})e^{-c_i(\ell-1)},
\qquad \ell=1,\ldots,Z.
\label{eq:oracle-discrete-marginal}
\end{equation}
These marginal gains are nonnegative and nonincreasing in $\ell$. Starting from $z_i=0$ for every prompt, the oracle repeatedly assigns one unit to an unsaturated prompt with the largest currently available gain $\Delta_i(z_i+1)$. Repeating this operation for $B$ units produces an optimal integer exposure vector, allowing zero-gain assignments when necessary.

The optimality follows from the decreasing marginal-gain sequences. A set of the $B$ largest gains can be chosen to contain a prefix of each prompt's sequence, with ties resolved in favor of earlier units. The greedy procedure constructs such a set while respecting $z_i\le Z$. A max-heap containing the next available marginal gain for each unsaturated prompt computes the exposure vector in $O(M+B\log(M+1))$ time and $O(M)$ auxiliary memory.

\paragraph{Recovery of a feasible rollout plan.}
To realize any feasible integer exposure vector $z$, let $T=\sum_i z_i$ and list its $T$ allocation units, placing the units for each prompt in one consecutive block. Assign the $r$-th unit in this list to round $1+((r-1)\bmod K)$. Each prompt's units are thereby spread as evenly as possible across rounds, and the same is true of the total allocation. The resulting counts satisfy
\begin{equation}
\sum_{k=1}^{K}N_{i,k}=z_i,
\qquad
N_{i,k}\le
\left\lceil\frac{z_i}{K}\right\rceil
\le N_{\max},
\qquad
\sum_iN_{i,k}\le
\left\lceil\frac{T}{K}\right\rceil
\le C.
\label{eq:oracle-round-feasibility}
\end{equation}
The final inequalities use $z_i\le KN_{\max}$, $T\le B\le KC$, and the integrality of $N_{\max}$ and $C$. Thus, the exposure-level optimum is attainable under the original global-budget and per-round capacity constraints.

\paragraph{Full-capacity boundary.}
When $B=KMN_{\max}$ and $C=MN_{\max}$, assigning $N_{i,k}=N_{\max}$ to every prompt in every round is feasible and yields $z_{i,K}=Z$. Since each utility is nondecreasing in exposure, this allocation attains
\[
\mathrm{OPT}^{\mathrm{stat}}_K(q)
=
\sum_{i=1}^{M}U(q_i,Z).
\]
The full-allocation rule therefore achieves zero fixed-rate regret without estimating the target rates.

The oracle makes the diminishing-return structure explicit: additional units are assigned according to their remaining marginal utility rather than a prompt's initial rate alone. It characterizes the fixed-rate benchmark, while CERO uses reward-derived proxy states and online dual updates to allocate rollouts as feedback becomes available.

\section{Proof of the Fixed-Rate Guarantee}
\label{app:static-proof}
For the analysis, if the execution budget is exhausted before round \(K\), we continue the virtual iterates through the remaining rounds under the same update rules, with all subsequent executed allocations set to zero.

For each round $k$, let $\widehat U_{i,k+1}$ be the post-feedback proxy utility and $\widehat U_{i,k+1}^*$ its compact conjugate. Let $\widetilde{\bm N}_k$ be the virtual proposal and $\widetilde A_k=\sum_i\widetilde N_{i,k}$. Work with the unnormalized duals $\theta_i^k=\cbar\vartheta_i^k$ and $\mu^k=\cbar\nu^k$. The normalized updates in \eqref{eq:theta-mu-update-normalized} are equivalent to projected subgradient descent with
\begin{equation}
\alpha_\theta=\frac{\cbar}{\Nmax\sqrt K},
\qquad
\alpha_\mu=\frac{\cbar}{G_B\sqrt K}.
\label{eq:unnormalized-steps}
\end{equation}

Define convex losses, after the round-$k$ proposal and proxy utilities have been realized, by
\begin{align}
h_k(\bm\theta)
&:=\sum_{i=1}^M\left[
\theta_i\widetilde N_{i,k}-\frac1K\widehat U_{i,k+1}^*(\theta_i)
\right],\\
m_k(\mu)
&:=\mu(\bar B-\widetilde A_k).
\label{eq:dual-losses}
\end{align}
Any $\zhat_{i,k+1}^*(\theta_i^k)$ is a supergradient of the concave conjugate at $\theta_i^k$, so
\begin{equation}
g_{\theta_i}^k=\widetilde N_{i,k}-\frac1K\zhat_{i,k+1}^*(\theta_i^k)
\end{equation}
is a subgradient of $h_k$. Also $g_\mu^k=\bar B-\widetilde A_k$ is the gradient of $m_k$. The gradient bounds are
\begin{equation}
\norm{\bm g_\theta^k}_2\le\sqrt M\Nmax,
\qquad
|g_\mu^k|\le G_B.
\end{equation}
The domain diameters are $\cbar\sqrt M$ and $\cbar$. Standard projected-OGD analysis therefore gives
\begin{align}
\sum_{k=1}^Kh_k(\bm\theta^k)
-\min_{\bm\theta\in[0,\cbar]^M}\sum_{k=1}^Kh_k(\bm\theta)
&\le \cbar M\Nmax\sqrt K,\\
\sum_{k=1}^Km_k(\mu^k)
-\min_{\mu\in[0,\cbar]}\sum_{k=1}^Km_k(\mu)
&\le \cbar G_B\sqrt K.
\label{eq:ogd-bounds}
\end{align}
These inequalities are pathwise; the losses may depend on the realized feedback and current proposal.

Define
\begin{equation}
\widetilde z_i:=\sum_{k=1}^K\widetilde N_{i,k},
\qquad
\widetilde T:=\sum_i\widetilde z_i,
\qquad
\widetilde V:=[\widetilde T-B]_+.
\end{equation}
Let
\begin{align}
\overline U_i^*(\theta)&:=\frac1K\sum_{k=1}^K\widehat U_{i,k+1}^*(\theta),\\
\overline U_i(z)&:=\min_{0\le\theta\le\cbar}\{\theta z-\overline U_i^*(\theta)\}.
\end{align}
Separability gives
\begin{align}
\min_{\bm\theta}\sum_{k=1}^Kh_k(\bm\theta)
&=\sum_{i=1}^M\overline U_i(\widetilde z_i),\\
\min_{0\le\mu\le\cbar}\sum_{k=1}^Km_k(\mu)
&=-\cbar\widetilde V.
\end{align}
Hence
\begin{equation}
\sum_{k=1}^K[h_k(\bm\theta^k)+m_k(\mu^k)]
\le
\sum_i\overline U_i(\widetilde z_i)
+\cbar(M\Nmax+G_B)\sqrt K
-\cbar\widetilde V.
\label{eq:ogd-upper}
\end{equation}

Let $\mathbf N^*$ attain \eqref{eq:static-opt}, define $z_i^*=\sum_kN_{i,k}^*$, and set $n_i^*=z_i^*/K$. Then $0\le n_i^*\le\Nmax$ and
\begin{equation}
\sum_i n_i^*=\frac1K\sum_i z_i^*\le\frac BK=\bar B\le C.
\end{equation}
Thus $\bm n^*\in\mathcal P_C$. Since $\mathcal P_C$ is integral, the integer proposal maximizes the same linear function over the continuous polytope. Therefore
\begin{equation}
\sum_i(\theta_i^k-\mu^k)\widetilde N_{i,k}
\ge
\sum_i(\theta_i^k-\mu^k)n_i^*.
\end{equation}
Using $\mu^k\ge0$, $\sum_i n_i^*\le\bar B$, and the proxy Fenchel representation yields
\begin{align}
h_k(\bm\theta^k)+m_k(\mu^k)
&\ge
\frac1K\sum_i\left[
\theta_i^kz_i^*-\widehat U_{i,k+1}^*(\theta_i^k)
\right]\\
&\ge\frac1K\sum_i\widehat U_{i,k+1}(z_i^*).
\label{eq:round-lower}
\end{align}
Summing and combining with \eqref{eq:ogd-upper} gives
\begin{equation}
\sum_i\frac1K\sum_k\widehat U_{i,k+1}(z_i^*)
\le
\sum_i\overline U_i(\widetilde z_i)
+\cbar(M\Nmax+G_B)\sqrt K
-\cbar\widetilde V.
\label{eq:proxy-comparison}
\end{equation}

Let $e_{i,k}=|\qhat_{i,k+1}-q_i|$ and
\begin{equation}
\delta_{i,k}:=\frac{4\lambda Z}{R}e_{i,k},
\qquad
\Delta:=\frac1K\sum_{i,k}\delta_{i,k}.
\end{equation}
By \Cref{lem:fenchel-stability}, both the proxy utility and its conjugate differ from their target versions by at most $\delta_{i,k}$. Consequently,
\begin{equation}
\OPT_K^{\mathrm{stat}}(\bm q)
-\sum_iU(q_i,\widetilde z_i)
+\cbar\widetilde V
\le
\cbar(M\Nmax+G_B)\sqrt K+2\Delta.
\label{eq:virtual-master}
\end{equation}
Since $Z=K\Nmax$,
\begin{equation}
\Delta
=\frac{4\lambda\Nmax}{R}\sum_{i,k}e_{i,k}
=\frac{4\lambda M\Nmax}{\bar N}\bar e_K^{\mathrm{stat}}.
\label{eq:delta-e}
\end{equation}

It remains to pass from virtual proposals to executed allocations. The prefix clipping rule implies
\begin{equation}
\sum_i(\widetilde z_i-z_{i,K}^{CERO})=\widetilde V.
\end{equation}
Every target utility has derivative at most $\cbar$, so
\begin{equation}
0\le\sum_i[U(q_i,\widetilde z_i)-U(q_i,z_{i,K}^{CERO})]
\le\cbar\widetilde V.
\end{equation}
This clipping loss is exactly absorbed by the violation term on the left of \eqref{eq:virtual-master}. Therefore
\begin{equation}
\Reg_K^{\mathrm{stat}}
\le
\cbar(M\Nmax+G_B)\sqrt K+2\Delta.
\end{equation}
Substitute $\cbar=\lambda/R$, $R=K\bar N$, use \eqref{eq:delta-e}, and divide by $M$. Since every prompt utility lies in $[0,1]$, the average regret is also at most one. Combining the two bounds gives \eqref{eq:static-bound}.

\section{Proof of the Same-Path Guarantee}
\label{app:dynamic-proof}
We first record a stability inequality for one incremental utility.

\begin{lemma}[Increment stability]
For any $q,q'\in[0,1/4]$ and $z,n\ge0$ satisfying $z+n\le Z$,
\begin{equation}
\left|
[U(q,z+n)-U(q,z)]-[U(q',z+n)-U(q',z)]
\right|
\le\frac{4\lambda}{R}n|q-q'|.
\label{eq:increment-stability}
\end{equation}
\end{lemma}

\begin{proof}
Let $\kappa=4\lambda/R$ and
\begin{equation}
G(q;z,n)=\int_z^{z+n}\kappa q e^{-\kappa qs}\,ds.
\end{equation}
Then
\begin{equation}
\partial_qG(q;z,n)
=\int_z^{z+n}\kappa(1-\kappa qs)e^{-\kappa qs}\,ds.
\end{equation}
Because $|(1-x)e^{-x}|\le1$ for $x\ge0$, the derivative magnitude is at most $\kappa n$. Apply the mean-value theorem.
\end{proof}

For each prompt choose the midpoint of its realized rate range,
\begin{equation}
q_i^\circ:=\frac12\left(\min_kq_{i,k}+\max_kq_{i,k}\right).
\end{equation}
Then $|q_{i,k}-q_i^\circ|\le V_K/2$ for every $i,k$. For any feasible allocation $\mathbf N$, let
\begin{align}
J^{\mathrm{dyn}}(\mathbf N)
&:=\sum_{i,k}\Delta U_{i,k}(\mathbf N;q_{i,k}),\\
J^\circ(\mathbf N)
&:=\sum_iU(q_i^\circ,z_{i,K}).
\end{align}
Using \eqref{eq:increment-stability} and total budget feasibility,
\begin{align}
|J^{\mathrm{dyn}}(\mathbf N)-J^\circ(\mathbf N)|
&\le\frac{4\lambda}{R}\sum_{i,k}N_{i,k}|q_{i,k}-q_i^\circ|\\
&\le\frac{2\lambda V_K}{R}B
=2\lambda M V_K.
\label{eq:dynamic-reference-gap}
\end{align}
Applying this inequality to the dynamic optimum and to the executed CERO allocation gives
\begin{equation}
\Reg_K^{\mathrm{dyn}}
\le
\OPT_K^{\mathrm{stat}}(\bm q^\circ)
-\sum_iU(q_i^\circ,z_{i,K}^{CERO})
+4\lambda M V_K.
\label{eq:dynamic-decomp}
\end{equation}
The average proxy error relative to $\bm q^\circ$ satisfies
\begin{align}
\frac1{MK}\sum_{i,k}|\qhat_{i,k+1}-q_i^\circ|
&\le\bar e_K^{\mathrm{dyn}}
+\frac1{MK}\sum_{i,k}|q_{i,k}-q_i^\circ|\\
&\le\bar e_K^{\mathrm{dyn}}+\frac{V_K}{2}.
\end{align}
Apply \Cref{thm:static} with reference vector $\bm q^\circ$ and substitute the last inequality into \eqref{eq:dynamic-decomp}. In addition, \eqref{eq:slope-bound} and total budget feasibility imply that every same-path objective is at most $\lambda M$, so the normalized regret is at most $\lambda$. Dividing by $M$ and combining the two bounds yields \eqref{eq:dynamic-bound}.

\section{Proofs of the Beta-state Propositions}
\label{app:beta}
\begin{proof}[Proof of \Cref{prop:group-variance}]
Conditional on $P_{i,k}$, each pair of future Bernoulli rewards disagrees with probability $2P_{i,k}(1-P_{i,k})$. The unbiased sample variance for binary data satisfies
$$
2\widehat{\mathrm{Var}}_G(r)
=
\binom{G}{2}^{-1}
\sum_{1\le j<\ell\le G}
\mathbf{1}\{r^{(j)}\ne r^{(\ell)}\}.
$$
Taking conditional expectations first given $P_{i,k}$ and then given $\mathcal H_i^k$ proves the claim.
\end{proof}

\begin{proof}[Proof of \Cref{prop:beta-calibration}]
Let $t=n+s_0$ and $\widehat m=a/t$. Then
\begin{equation}
\qhat_n=\frac{t}{t+1}\widehat m(1-\widehat m).
\end{equation}
The function $x(1-x)$ is one-Lipschitz on $[0,1]$, so
\begin{equation}
|\qhat_n-p(1-p)|
\le |\widehat m-p|+\frac1{4(t+1)}.
\end{equation}
Moreover,
\begin{equation}
|\E\widehat m-p|
=\frac{|a_0-ps_0|}{t}\le\frac{s_0}{t},
\end{equation}
and
\begin{equation}
\E|\widehat m-\E\widehat m|
\le\sqrt{\operatorname{Var}(\widehat m)}
=\frac{\sqrt{np(1-p)}}{t}
\le\frac{\sqrt n}{2t}.
\end{equation}
Combining the bounds proves \eqref{eq:beta-calibration-bound}.
\end{proof}

\section{Implementation Details}
\label{sec:implem}

\paragraph{Training.} We use two locally modified veRL-based \citep{sheng2025hybridflow} stacks with matched learner and decoding hyperparameters, vLLM 0.8.3 for generation, and FSDP for actor training. Bfloat16 is used for generation and actor computation. Experiments use two NVIDIA H800 GPUs for the 1.5B backbone and four for the 4B and 7B backbones. All methods use GRPO with within-group standard-deviation normalization, advantage clipping to $[-5,5]$, and loss averaged over valid response tokens. AdamW uses a constant learning rate of $10^{-6}$, $(\beta_1,\beta_2)=(0.9,0.999)$, zero weight decay, and no warmup. We use one PPO epoch per collected batch, no reshuffling, and a global PPO mini-batch size of 32. The lower/upper policy clipping parameters are $0.20/0.28$, the dual-clip constant is $10$, the importance-ratio cap is $2$, and the gradient-norm threshold is $1$. KL and entropy regularization are disabled. Training prompts are left-truncated to 2,048 tokens and responses are capped at 2,048 tokens; sampling uses temperature $1.0$, top-p $1.0$, and top-k $-1$. Training rewards are computed using the rule-based mathematical verifier from DAPO \citep{yu2025dapo}, which treats $+1$ as success and $0$ as failure. The CERO run in Table~\ref{tab:main} is included in the multi-seed results in Table~\ref{tab:ablation}.

A scheduler round is one allocation opportunity, whereas a training step is one policy-update call on a collected rollout batch. Zero-allocation rounds advance the scheduler without updating the policy. A data epoch is one traversal of the training data loader, distinct from a PPO epoch over a collected rollout batch.

\paragraph{Allocation settings.} We use Algorithm 1 with $g=4$ responses per allocation unit and $N_{\max}=2$, $C=132$, and $B=64{,}000$ in allocation units. CERO uses a $\mathrm{Beta}(1,1)$ prior, $\lambda=\rho=1$, $\nu^1=0.5$, and $\gamma_{\vartheta}=\gamma_{\nu}=1/\sqrt{500}$. Ties use a fixed seed permutation. Reserve-aware catch-up was enabled but never activated in the reported CERO runs. KnapsackRL uses positive and negative base allocations of two, a reward-history buffer of length 16, a minimum initialized history length of four, and robust epoch estimation.

CERO-greedy uses the same reward-proxy estimator, utility function, and fixed 512-response quota as CERO-fixed, but replaces prompt-dual scores with the marginal utility of one additional complete group. In our experiments, each group contains eight responses, corresponding to $N_{\max}=2$ allocation units. At the beginning of round $k$, the selection score for prompt $i$ is
\[
\Delta_{i,k}^{\mathrm{grp}}
:=
U\!\left(\hat q_{i,k},z_{i,k-1}+N_{\max}\right)
-
U\!\left(\hat q_{i,k},z_{i,k-1}\right),
\]
where $\hat q_{i,k}$ is the pre-decision proxy and $z_{i,k-1}$ is the cumulative executed exposure in the control run. We use CERO's exponential utility with the same saturation parameter and exposure normalization. Each round scores the full training prompt pool and selects the $64$ prompts with the largest scores, generating one eight-response group per prompt. This control preserves the reward-based utility and its diminishing-return structure while replacing dual-based selection with direct marginal-utility ranking.

\paragraph{Evaluation.} We follow \citet{tan2025deepscaler} and evaluate the checkpoint obtained after $500$ scheduler rounds using 16 responses per problem, temperature $0.6$, top-p $1.0$, and top-k $-1$. All methods use the same rule-based mathematical equivalence checker without an LLM judge. Metric definitions are given in Section ~\ref{sec:experiments} and Appendix ~\ref{app:best16_eval}; computational costs and schedule audits are reported in Appendix ~\ref{app:computational_cost}.

\end{APPENDICES}

\end{document}